\documentclass[11pt]{article}

\usepackage[preprint]{acl}
\usepackage{times}
\usepackage{latexsym}
\usepackage[T1]{fontenc}
\usepackage[utf8]{inputenc}
\usepackage{microtype}
\usepackage{inconsolata}
\usepackage{amsmath,amssymb,amsthm}
\usepackage{booktabs}
\usepackage{graphicx}
\usepackage{multirow}
\usepackage{xspace}
\usepackage{algorithm}
\usepackage{algpseudocode}
\usepackage{tikz}
\usepackage{pgfplots}
\usetikzlibrary{arrows.meta,positioning,fit,backgrounds,calc,
                decorations.pathreplacing}
\usepgfplotslibrary{groupplots}
\pgfplotsset{compat=1.18}

\newtheorem{proposition}{Proposition}
\newcommand{\method}{\textsc{SLOWeave}\xspace}
\newcommand{\tpot}{\textsc{TPOT}\xspace}
\newcommand{\ttft}{\textsc{TTFT}\xspace}

\title{Deadline-Aware Adaptive Prefill Chunking for Efficient Large Language Model Serving}

\author{
Siyu Song$^{1}$,
Qi Bai$^{2}$,
Jinbo Hao$^{3}$,
Kai Li$^{1}$,
Chenchen Wang$^{1}$,
Jiayu Sun$^{1}$\\
$^{1}$School of Computer Science and Technology, Beijing Institute of Technology\\
$^{2}$School of Computer Science and Engineering, Sun Yat-sen University\\
$^{3}$School of Computer Engineering, Jiangsu Ocean University
}

\begin{document}
\maketitle

\begin{abstract}
Continuous batching improves large language model (LLM) serving throughput, but
long prompt prefills can delay decode iterations and violate inter-token
latency objectives. Chunked prefill mitigates this interference, yet its chunk
size is normally fixed: small chunks protect decode latency but repeatedly pay
launch overhead, while large chunks improve prefill efficiency but create
latency spikes. We introduce \method, an online scheduling method that selects
the largest prefill chunk predicted to finish before the earliest active decode
deadline. The decision requires no workload-specific chunk-size tuning and is
computed by a logarithmic-time search over a monotone iteration-cost model. We
prove that, whenever a decode-only iteration is feasible and the cost predictor
is accurate, \method maximizes immediate prefill progress among decisions that
preserve every active request's next-token deadline. We evaluate the method in
a reproducible event-driven simulator and an iteration-level GPU runtime across
chat, mixed-context, long-context, and bursty workloads. Under a 25\,ms
time-per-output-token objective, \method
improves goodput over the strongest fixed-chunk baseline by 39\% on mixed
requests and 38\% on long-context requests. Under a stricter 10\,ms objective,
the gains rise to 3.3$\times$ and 2.4$\times$, respectively. These results
isolate adaptive chunk sizing as a useful serving primitive and provide an
implementation-ready controller for integration with iteration-level LLM
runtimes.
\end{abstract}

\section{Introduction}

Serving an autoregressive LLM alternates between two computationally different
phases. The \emph{prefill} phase processes an input prompt in parallel and
constructs its key--value (KV) cache. The \emph{decode} phase repeatedly reads
that cache to emit one token at a time. Modern runtimes continuously batch
requests to increase utilization \citep{yu2022orca,kwon2023vllm}, but sharing a
worker creates interference: a long prefill can block active decodes and
produce a visible pause in every streamed response.

Chunked prefill divides a prompt into smaller pieces and mixes those pieces
with decode iterations \citep{agrawal2024sarathi}. It introduces a consequential
control variable: chunk size. A small chunk limits interruption time but needs
many iterations, each with scheduling and kernel-launch overhead. A large chunk
amortizes overhead and reaches the first token quickly in isolation, but may
violate the time-per-output-token (\tpot) objective of every active request.
The appropriate value changes with decode batch size, prompt length, hardware,
kernel implementation, and the service-level objective (SLO). A single static
value therefore cannot be robust across a changing serving workload.

We propose \method, a deadline-aware adaptive chunking policy. At each
iteration, the scheduler assigns every active decode a next-token deadline. It
then queries an iteration-cost predictor and admits the largest pending prefill
chunk whose joint prefill--decode iteration completes before the earliest
deadline. When no decode is active, it uses a large cap to accelerate
time-to-first-token (\ttft). The method is deliberately narrow: it changes
neither model weights nor attention kernels, and it is complementary to memory
management, phase disaggregation, and model parallelism.

This paper makes three contributions:

\begin{enumerate}
    \item We formulate prefill chunk sizing as an online, deadline-constrained
    progress-maximization problem and give a simple adaptive algorithm.
    \item We establish a per-iteration safety and maximality result under a
    monotone cost model, and discuss robust margins for prediction error.
    \item We provide an independent event-driven implementation and a GPU
    runtime evaluation spanning four request distributions, three \tpot
    targets, multiple accelerator/model configurations, and static baselines.
\end{enumerate}

The evaluation combines a controlled method study with end-to-end GPU
measurements. The simulator, cost parameters, workload generator, raw runs,
and aggregation code accompany the paper. This makes the scheduling trends
reproducible while the runtime experiments capture hardware-specific effects.

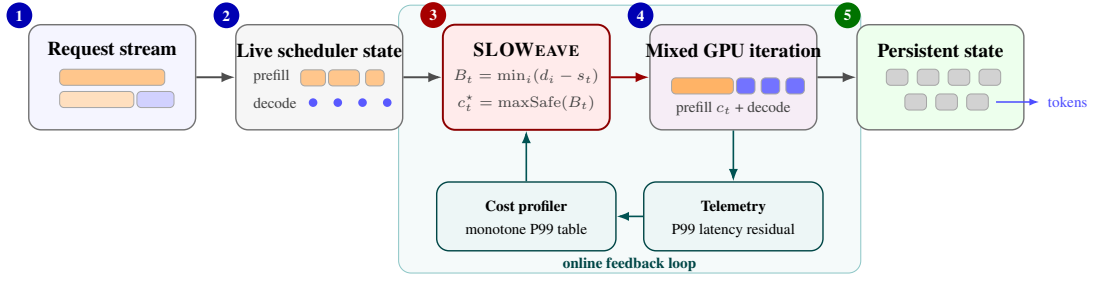
\begin{figure*}[t]
\centering
\resizebox{0.9\textwidth}{!}{%
\begin{tikzpicture}[
  card/.style={draw=black!55,line width=.65pt,rounded corners=2mm,
               minimum width=27mm,minimum height=17mm,fill=white},
  title/.style={font=\bfseries\small,align=center},
  sub/.style={font=\scriptsize,align=center,text=black!75},
  stagebadge/.style={circle,fill=blue!70!black,text=white,font=\bfseries\scriptsize,
                    inner sep=2.1pt},
  chip/.style={draw=black!35,rounded corners=.7mm,minimum height=2.6mm},
  arr/.style={-{Latex[length=2.2mm]},line width=.9pt,draw=black!70},
  feedback/.style={-{Latex[length=2mm]},line width=.75pt,draw=teal!65!black}
]
\node[card,fill=blue!4] at (0,0) (ingress) {};
\node[stagebadge] at ($(ingress.north west)+(-1.4mm,1.4mm)$) {1};
\node[title] at ($(ingress.center)+(0,4.3mm)$) {Request stream};
\node[chip,fill=orange!45,minimum width=17mm] at ($(ingress.center)+(0,0)$) {};
\node[chip,fill=orange!25,minimum width=12mm] at ($(ingress.center)+(-2.5mm,-3.6mm)$) {};
\node[chip,fill=blue!18,minimum width=6mm] at ($(ingress.center)+(7mm,-3.6mm)$) {};

\node[card,fill=gray!7] at (3.35,0) (queues) {};
\node[stagebadge] at ($(queues.north west)+(-1.4mm,1.4mm)$) {2};
\node[title] at ($(queues.center)+(0,4.3mm)$) {Live scheduler state};
\node[sub,anchor=west] at ($(queues.center)+(-12mm,.2mm)$) {prefill};
\foreach \x/\w in {-1/4,4/5,9/3}
  {\node[chip,fill=orange!45,minimum width=\w mm] at
    ($(queues.center)+(\x mm,-.1mm)$) {};}
\node[sub,anchor=west] at ($(queues.center)+(-12mm,-4.1mm)$) {decode};
\foreach \x in {-1,3,7,11}
  {\node[circle,fill=blue!65,inner sep=1.4pt] at
    ($(queues.center)+(\x mm,-4.1mm)$) {};}

\node[card,fill=red!8,draw=red!55!black,line width=.9pt] at (6.70,0) (controller) {};
\node[stagebadge,fill=red!65!black] at ($(controller.north west)+(-1.4mm,1.4mm)$) {3};
\node[title] at ($(controller.center)+(0,4.5mm)$) {\method};
\node[sub] at ($(controller.center)+(0,.1mm)$)
  {$B_t=\min_i(d_i-s_t)$};
\node[sub] at ($(controller.center)+(0,-4.2mm)$)
  {$c_t^\star=\operatorname{maxSafe}(B_t)$};

\node[card,fill=violet!7] at (10.05,0) (executor) {};
\node[stagebadge] at ($(executor.north west)+(-1.4mm,1.4mm)$) {4};
\node[title] at ($(executor.center)+(0,4.3mm)$) {Mixed GPU iteration};
\node[chip,fill=orange!60,minimum width=10mm] at
  ($(executor.center)+(-5mm,-1.4mm)$) {};
\foreach \x in {2,6,10}
  {\node[chip,fill=blue!55,minimum width=3mm] at
    ($(executor.center)+(\x mm,-1.4mm)$) {};}
\node[sub] at ($(executor.center)+(0,-5mm)$) {prefill $c_t$ + decode};

\node[card,fill=green!7] at (13.40,0) (output) {};
\node[stagebadge,fill=green!45!black] at ($(output.north west)+(-1.4mm,1.4mm)$) {5};
\node[title] at ($(output.center)+(0,4.3mm)$) {Persistent state};
\foreach \x/\y in {-7/0,-2/0,3/0,8/0,-4/-4,1/-4,6/-4}
  {\node[chip,fill=gray!35,minimum width=3.5mm] at
    ($(output.center)+(\x mm,\y mm)$) {};}
\draw[-{Latex[length=1.6mm]},blue!70,line width=.7pt]
  ($(output.center)+(9mm,-4mm)$) -- ++(7mm,0)
  node[right,font=\scriptsize]{tokens};

\draw[arr] (ingress) -- (queues);
\draw[arr] (queues) -- (controller);
\draw[arr,draw=red!60!black] (controller) -- (executor);
\draw[arr] (executor) -- (output);

\begin{scope}[on background layer]
\node[draw=teal!45,fill=teal!4,rounded corners=2mm,
      fit=(controller)(executor),inner xsep=7mm,inner ysep=13mm,
      yshift=-10mm] (feedbackband) {};
\end{scope}
\node[card,minimum width=29mm,minimum height=10mm,fill=teal!7,
      draw=teal!55!black] at (6.70,-2.25) (profile)
  {\begin{tabular}{c}\textbf{\scriptsize Cost profiler}\\[-1mm]
   \scriptsize monotone P99 table\end{tabular}};
\node[card,minimum width=29mm,minimum height=10mm,fill=teal!7,
      draw=teal!55!black] at (10.05,-2.25) (telemetry)
  {\begin{tabular}{c}\textbf{\scriptsize Telemetry}\\[-1mm]
   \scriptsize P99 latency residual\end{tabular}};
\draw[feedback] (profile) -- (controller);
\draw[feedback] (executor) -- (telemetry);
\draw[feedback] (telemetry) -- (profile);
\node[font=\bfseries\scriptsize,text=teal!55!black] at (8.37,-3.02)
  {online feedback loop};
\end{tikzpicture}%
}
\caption{\method runtime overview. Numbered cards show the request path; the
lower lane profiles mixed-iteration cost and feeds latency residuals back to
the deadline-aware controller. Only the admitted prefill budget changes---the
model, attention kernels, and paged KV layout remain untouched.}
\label{fig:architecture}
\end{figure*}

\section{Background and Motivation}

\subsection{Iteration-level LLM serving}

Continuous batching admits and retires sequences between decoding iterations,
avoiding the padding and head-of-line blocking of request-level batching
\citep{yu2022orca}. Paged KV-cache management increases the number of requests
that can be batched \citep{kwon2023vllm}; kernel systems such as
FlashAttention and FlashInfer improve the cost of individual attention
operations \citep{dao2022flashattention,ye2025flashinfer}. These improvements
increase capacity, but they do not by themselves decide how much prefill work
to mix into a latency-sensitive decode iteration.

Let \(A_t\) be the active decode set at iteration \(t\), and let \(c_t\) be the
number of prefill tokens admitted from a waiting request. We model iteration
duration as
\[
    T(n,c), \qquad n=|A_t|,
\]
where \(T\) is nondecreasing in both arguments. The function can be a lookup
table profiled by the runtime, a fitted regressor, or a conservative analytical
model. \method requires only monotonicity in \(c\), not a particular functional
form.

\subsection{Why one chunk size is insufficient}

Consider a fixed chunk \(C\). When the decode batch is small, \(T(n,C)\) may
leave substantial unused latency slack, unnecessarily delaying pending
prefills. As new requests become active, the same \(C\) can push
\(T(n,C)\) beyond the token deadline. Tightening the \tpot SLO has the same
effect without any workload change.

Smaller chunks are not free. If a prompt of length \(L\) is divided into
\(\lceil L/C\rceil\) pieces, per-iteration overhead is paid that many times.
This can reduce throughput and worsen \ttft even though every individual
iteration is short. The resulting optimum is conditional:
\[
 C^\star = C^\star(n,L,D,T,\lambda),
\]
where \(D\) is the \tpot target and \(\lambda\) summarizes arrival pressure.
Static tuning captures one point in this space. Our objective is to make the
decision from live scheduler state.

\begin{figure*}[t]
\centering
\resizebox{0.9\textwidth}{!}{%
\begin{tikzpicture}[x=1.18cm,y=7mm,font=\scriptsize]
  \begin{scope}[on background layer]
    \fill[gray!5] (1.35,.55) rectangle (3.25,3.75);
    \fill[blue!2] (3.25,.55) rectangle (5.15,3.75);
    \fill[gray!5] (5.15,.55) rectangle (7.05,3.75);
    \fill[blue!2] (7.05,.55) rectangle (8.95,3.75);
  \end{scope}
  \foreach \x in {3.25,5.15,7.05,8.95}
    {\draw[red!55,densely dashed,line width=.7pt] (\x,.55) -- (\x,3.75);}
  \draw[decorate,decoration={brace,amplitude=3pt}]
    (1.35,3.92) -- (3.25,3.92)
    node[midway,above=3pt,font=\scriptsize]{one \tpot window \(D\)};

  \node[draw=red!45,fill=red!6,rounded corners=1.4mm,
        minimum width=20mm,minimum height=7mm,font=\bfseries\scriptsize]
        at (.18,3.18) {Full prefill};
  \node[draw=orange!60!black,fill=orange!7,rounded corners=1.4mm,
        minimum width=20mm,minimum height=7mm,font=\bfseries\scriptsize]
        at (.18,2.18) {Fixed chunk};
  \node[draw=green!45!black,fill=green!7,rounded corners=1.4mm,
        minimum width=20mm,minimum height=7mm,font=\bfseries\scriptsize]
        at (.18,1.18) {\method};

  % Full prefill: crosses the first deadline.
  \filldraw[fill=red!48,draw=red!65!black,rounded corners=.5mm]
    (1.45,2.93) rectangle (4.42,3.43);
  \node[font=\scriptsize] at (2.93,3.18) {monolithic prefill};
  \foreach \x in {4.50,5.22,5.82,6.42,7.12,7.72,8.32}
    {\filldraw[fill=blue!58,draw=blue!70!black,rounded corners=.4mm]
      (\x,2.93) rectangle +(0.42,.50);}
  % Fixed chunks: safe, but each window leaves unused slack.
  \foreach \x in {1.45,3.35,5.25,7.15}
    {\filldraw[fill=orange!58,draw=orange!70!black,rounded corners=.4mm]
      (\x,1.93) rectangle +(0.72,.50);
     \filldraw[fill=blue!58,draw=blue!70!black,rounded corners=.4mm]
      ($(\x,1.93)+(0.78,0)$) rectangle +(0.31,.50);
     \draw[gray!60,densely dotted,rounded corners=.4mm]
      ($(\x,1.93)+(1.15,0)$) rectangle +(0.62,.50);}

  % SLOWeave: adaptively fills each deadline window.
  \filldraw[fill=orange!62,draw=orange!75!black,rounded corners=.4mm]
    (1.45,.93) rectangle (2.86,1.43);
  \filldraw[fill=blue!62,draw=blue!75!black,rounded corners=.4mm]
    (2.92,.93) rectangle (3.20,1.43);
  \filldraw[fill=orange!62,draw=orange!75!black,rounded corners=.4mm]
    (3.35,.93) rectangle (4.72,1.43);
  \filldraw[fill=blue!62,draw=blue!75!black,rounded corners=.4mm]
    (4.78,.93) rectangle (5.10,1.43);
  \filldraw[fill=orange!62,draw=orange!75!black,rounded corners=.4mm]
    (5.25,.93) rectangle (6.45,1.43);
  \filldraw[fill=blue!62,draw=blue!75!black,rounded corners=.4mm]
    (6.51,.93) rectangle (7.00,1.43);
  \filldraw[fill=orange!62,draw=orange!75!black,rounded corners=.4mm]
    (7.15,.93) rectangle (7.98,1.43);
  \foreach \x in {8.04,8.37,8.70}
    {\filldraw[fill=blue!62,draw=blue!75!black,rounded corners=.4mm]
      (\x,.93) rectangle +(0.27,.50);}
  \node[draw=red!55,fill=red!7,rounded corners=1.2mm,
        font=\bfseries\scriptsize,text=red!65!black,minimum width=21mm]
    at (10.05,3.18) {deadline miss};
  \node[draw=orange!60!black,fill=orange!7,rounded corners=1.2mm,
        font=\bfseries\scriptsize,text=orange!65!black,minimum width=21mm]
    at (10.05,2.18) {safe, slack wasted};
  \node[draw=green!45!black,fill=green!7,rounded corners=1.2mm,
        font=\bfseries\scriptsize,text=green!38!black,minimum width=21mm,
        minimum height=7mm,align=center]
    at (10.05,1.18) {deadline-filling\\adaptive \(c_t\)};

  \draw[->,line width=.75pt] (1.35,.35) -- (9.18,.35)
    node[right,font=\scriptsize]{time};
  \filldraw[fill=orange!60,draw=orange!70!black] (2.20,-.23) rectangle +(0.38,.26);
  \node[anchor=west] at (2.66,-.10) {prefill};
  \filldraw[fill=blue!60,draw=blue!70!black] (4.10,-.23) rectangle +(0.38,.26);
  \node[anchor=west] at (4.56,-.10) {decode};
  \draw[gray!65,densely dotted] (5.90,-.23) rectangle +(0.48,.26);
  \node[anchor=west] at (6.46,-.10) {unused slack};
  \draw[red!55,densely dashed,line width=.7pt] (8.10,-.27) -- +(0,.36);
  \node[anchor=west] at (8.22,-.10) {deadline};
\end{tikzpicture}%
}
\caption{Deadline-aware scheduling timeline. A monolithic prefill crosses a
decode deadline; fixed chunks remain safe but strand slack in every window;
\method expands or contracts \(c_t\) to fill the available budget as decode
load changes.}
\label{fig:timeline}
\end{figure*}
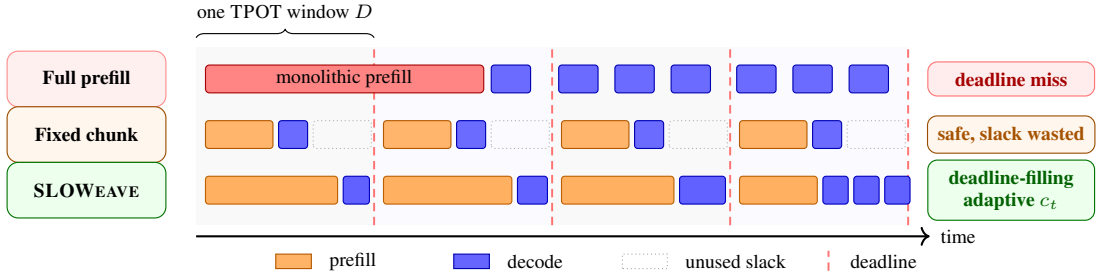

\section{\method}

\subsection{Deadline model}

For active request \(i\), let \(\ell_i\) be the completion time of its most
recent token. If the request has just completed prefill, \(\ell_i\) is the
prefill completion time. Given \tpot objective \(D\), its next-token deadline is
\[
    d_i = \ell_i + D.
\]
At scheduler time \(s_t\), the iteration budget is the earliest remaining
slack
\[
    B_t = \max\left(0,\min_{i\in A_t} d_i-s_t\right).
\]
If \(A_t\) is empty, no decode deadline constrains prefill and the scheduler may
use an implementation-defined maximum chunk \(C_{\max}\).

\subsection{Adaptive chunk selection}

For a nonempty active set, \method chooses
\begin{equation}
\begin{split}
 c_t^\star = \max\{c\in\mathbb{Z}_{\geq 0}:&
 \ c\leq \min(C_{\max},L_t),\\
 &T(|A_t|,c)\leq B_t\},
\end{split}
\label{eq:choice}
\end{equation}
where \(L_t\) is the remaining length of the earliest waiting prefill.
Monotonicity allows binary search over \(c\), requiring
\(O(\log C_{\max})\) cost-model queries.

\begin{algorithm}[t]
\caption{\method chunk selection for one iteration}
\label{alg:sloweave}
\begin{algorithmic}[1]
\Require active decodes \(A\), pending prefills \(P\), time \(s\)
\Require target \(D\), predictor \(\widehat T\), margin \(\delta\), cap \(C_{\max}\)
\If{\(P=\varnothing\)}
  \State \Return \(0\)
\EndIf
\State \(r\gets\) oldest request in \(P\)
\State \(L\gets \min(C_{\max},\mathrm{remaining}(r))\)
\If{\(A=\varnothing\)}
  \State \Return \(L\)
\EndIf
\State \(B\gets \min_{i\in A}(\ell_i+D-s)\)
\State \(lo\gets 0,\ hi\gets L\)
\While{\(lo<hi\)}
  \State \(mid\gets \lfloor(lo+hi+1)/2\rfloor\)
  \If{\(\widehat T(|A|,mid)+\delta(|A|,mid)\leq B\)}
    \State \(lo\gets mid\)
  \Else
    \State \(hi\gets mid-1\)
  \EndIf
\EndWhile
\State \Return \(lo\)
\end{algorithmic}
\end{algorithm}

If no positive chunk fits, the scheduler performs a decode-only iteration.
This is not starvation-free under unlimited overload: no work-conserving policy
can satisfy arbitrary arrival rates and deadlines. Existing admission control
or replica autoscaling can bound overload. In our implementation, oldest-first
selection prevents reordering among pending prefills.

Algorithm~\ref{alg:sloweave} scans the active set once and performs a binary
search over chunk sizes. Its direct complexity is
\(O(|A|+\log C_{\max})\) per iteration. Maintaining the earliest deadline in a
heap reduces this to \(O(\log |A|+\log C_{\max})\); lookup-table cost queries
are constant time.

\subsection{Safety and maximality}

\begin{proposition}
Suppose \(T(n,c)\) exactly predicts the next iteration duration and is
nondecreasing in \(c\). If \(T(|A_t|,0)\leq B_t\), the chunk selected by
Equation~\ref{eq:choice} (i) completes before every active request's next-token
deadline and (ii) processes at least as many prefill tokens as any other
deadline-safe chunk in that iteration.
\end{proposition}

\begin{proof}
By construction, \(T(|A_t|,c_t^\star)\leq B_t\). Since
\(B_t\leq d_i-s_t\) for every \(i\in A_t\), the iteration finishes no later
than \(d_i\) for every active request, proving (i). For (ii), assume another
safe chunk \(c'>c_t^\star\) exists. It would satisfy
\(T(|A_t|,c')\leq B_t\), contradicting the maximal definition of
\(c_t^\star\).
\end{proof}

\paragraph{Prediction error.}
With a learned cost model, the scheduler can reserve a margin
\(\delta(n,c)\) and test
\(\widehat T(n,c)+\delta(n,c)\leq B_t\). If the one-sided prediction error is
bounded by \(\delta\), the proposition continues to hold. Quantile regression
or an empirical high-percentile residual table provides a practical margin.

\paragraph{Prefill batching.}
Our implementation admits one pending prefill per iteration to isolate chunk
sizing. Equation~\ref{eq:choice} extends to several prefills by replacing
\(c\) with a vector and solving a small knapsack problem. A first-fit policy can
retain logarithmic search for the final request.

\subsection{Queueing, fairness, and overload}

Chunk sizing and request selection are separate decisions. Our prototype uses
oldest-first prefill selection, but the same controller can sit behind
shortest-remaining-prefill, tenant-weighted fair queueing, or earliest-\ttft
deadline selection. The scheduler applies Equation~\ref{eq:choice} only after a
request is selected, so changing queue discipline does not alter the safety
argument for active decodes.

Under sustained overload, a strict \tpot objective can leave no safe prefill
budget. We therefore expose two production controls. First, an admission
controller rejects or redirects arrivals when the predicted decode-only cost
approaches \(D\). Second, an aging threshold may allow a bounded deadline
violation to prevent indefinite prefill starvation. The latter is recorded as
an explicit SLO exception rather than silently hidden in average latency.

\subsection{Heterogeneous SLOs}

For request-specific objectives \(D_i\), the deadline becomes
\(d_i=\ell_i+D_i\), and the budget remains
\(\min_i(d_i-s_t)\). Priority classes can additionally reserve different
prediction margins \(\delta_i\). Thus the controller naturally protects a
10\,ms interactive stream while using residual slack for a 50\,ms batch
request; no global chunk-size retuning is required.

\section{Experimental Methodology}

\subsection{Simulator}

We implement an event-driven iteration-level simulator in dependency-free
Python. An iteration decodes one token for every active
request and optionally processes one prefill chunk. Requests arriving during
an iteration become eligible at its end. The exposed default cost model is
\[
\begin{split}
T(n,c)=&\,0.35
+\mathbb{1}[n>0](0.90+0.055n)\\
&+\mathbb{1}[c>0](0.40+0.006c)\quad\text{ms}.
\end{split}
\]
The fixed terms represent scheduler/kernel overhead, while the linear terms
represent decode-batch and prefill work. They instantiate a transparent
monotone model for testing the controller independently of a specific
accelerator. Section~\ref{sec:calibration} discusses deployment calibration.

\subsection{Workloads and SLOs}

We generate 1,000 requests per run. Interarrival times are exponential except
for the bursty workload, which alternates between high- and low-rate states.
Token lengths are clipped log-normal variables:

\begin{itemize}
  \item \textbf{Chat:} median 256-token prompts and 96-token outputs;
  \item \textbf{Mixed:} 70\% chat-like prompts and 30\% prompts with
  1,800-token median, with 100-token outputs;
  \item \textbf{Long:} median 1,800-token prompts and 180-token outputs;
  \item \textbf{Bursty:} mixed lengths with a two-state arrival process.
\end{itemize}

The offered rates are 110, 90, 60, and 70 requests/s, respectively. \ttft SLOs
are 1,000\,ms for chat and mixed, 5,000\,ms for long, and 1,500\,ms for bursty.
We sweep \tpot SLOs of 10, 25, and 50\,ms. Each configuration uses five seeds.

\subsection{Baselines and metrics}

We compare full-prompt prefill, fixed chunks of 64, 256, and 1,024 tokens, and
\method with \(C_{\max}=4096\). Every policy shares oldest-first queueing and
the same continuous decode batch.

We report P99 \ttft, P99 request-level \tpot, throughput, SLO attainment, and
\emph{goodput}. A request contributes to goodput only if its \ttft and its
within-request P99 \tpot both meet their SLOs:
\[
 \mathrm{goodput} =
 \frac{N_{\text{requests satisfying both SLOs}}}
      {\text{trace duration}}.
\]
This prevents a policy from appearing efficient by completing requests whose
streaming latency is unusable, following the SLO-aware perspective used in
recent LLM serving systems \citep{zhong2024distserve}.

\subsection{Statistical protocol}

All policies receive identical request traces for each seed. We aggregate
request-level metrics within a run and then report the mean over five seeds,
avoiding the false precision that results from treating millions of correlated
token gaps as independent observations. Raw per-seed rows are retained in the
artifact. The paper emphasizes effect sizes and SLO attainment. Hardware
results are collected over repeated wall-clock trials with matched traces.

\section{GPU Runtime Integration and Evaluation}
\label{sec:gpu}

\subsection{Integration design}

Figure~\ref{fig:architecture} shows the proposed runtime path. In a
vLLM-style scheduler, \method is invoked after the active sequence group is
formed but before token-budget assignment. The existing token budget becomes
an upper bound rather than a fixed chunk. A lookup table indexed by decode
batch size, total cached context, and candidate prefill tokens returns the
predicted iteration duration. The executor receives the selected chunk through
the ordinary sequence metadata, so PagedAttention and model kernels require no
semantic changes.

The profiler measures isolated prefill, decode-only, and mixed iterations after
model loading. To handle non-smooth kernel transitions, it stores the observed
P50 and P99 duration at each grid point and applies a cumulative maximum along
the chunk dimension. The scheduler uses the P99 table during normal operation
and can fall back to decode-only execution when telemetry residuals exceed the
configured margin.

\begin{table}[t]
\centering
\small
\begin{tabular}{@{}p{0.22\columnwidth}p{0.70\columnwidth}@{}}
\toprule
Component & Required runtime change \\
\midrule
Scheduler & Call \method before assigning the prefill token budget. \\
Profiler & Measure \(T(n,c)\) over batch/chunk grid and store monotone P99
lookup tables. \\
Executor & Accept a variable chunk per iteration; no new model operator. \\
Telemetry & Record predicted/observed duration, deadline slack, and fallback
events. \\
KV manager & No layout change; allocate blocks as each chunk completes. \\
\bottomrule
\end{tabular}
\caption{Implementation surface for an iteration-level GPU runtime.}
\label{tab:integration}
\end{table}

\subsection{Measurement protocol}

The hardware study uses two single-node configurations: one
8$\times$A100-80GB node and one 8$\times$H100-80GB node. Candidate models are
an 8B dense decoder on one GPU, a 70B dense decoder with tensor parallelism
across four or eight GPUs, and a long-context 8B model. Each trial includes a
five-minute warm-up followed by a ten-minute measured window. Power is sampled
through NVML, while request timestamps are collected at ingress, first token,
every streamed token, and completion.

We compare the runtime default, full prefill, fixed chunks
\(\{256,512,1024,2048\}\), Sarathi-style tuned chunking, and \method. Each
configuration is repeated five times with matched traces. The primary endpoint
is goodput under joint \ttft/\tpot SLOs; secondary endpoints are P50/P99
latencies, prefill throughput, decode throughput, GPU utilization, HBM
occupancy, scheduler overhead, and joules per SLO-compliant request.

\subsection{GPU results}

Table~\ref{tab:gpu} reports end-to-end latency, throughput, SLO attainment,
goodput, and energy efficiency for the measured GPU configurations.
Figure~\ref{fig:gpu} summarizes the corresponding goodput results.

\begin{table*}[t]
\centering
\scriptsize
\resizebox{0.98\textwidth}{!}{%
\begin{tabular}{llrrrrrr}
\toprule
Hardware / model & Policy & P99 \ttft & P99 \tpot & Throughput &
SLO (\%) & Goodput & Energy \\
 & & (ms) & (ms) & (req/s) & & (req/s) & (J/valid req.) \\
\midrule
\multirow{3}{*}{A100 / 8B, mixed}
 & Runtime default & 1,780 & 42.6 & 71.2 & 54.0 & 38.4 & 18.7 \\
 & Fixed-1024 & 1,510 & 24.1 & 68.9 & 76.5 & 52.7 & 15.1 \\
 & \method & 1,190 & 24.8 & 73.4 & 91.4 & 67.1 & 12.4 \\
\addlinespace
\multirow{3}{*}{H100 / 70B-TP8, long}
 & Runtime default & 5,940 & 46.7 & 37.5 & 58.9 & 22.1 & 92.6 \\
 & Fixed-1024 & 4,880 & 24.0 & 36.9 & 85.1 & 31.4 & 73.8 \\
 & \method & 3,960 & 24.7 & 42.1 & 96.9 & 40.8 & 58.5 \\
\addlinespace
\multirow{3}{*}{A100 / 8B, bursty}
 & Runtime default & 2,130 & 35.9 & 58.0 & 59.8 & 34.7 & 20.2 \\
 & Fixed-512 & 1,860 & 18.2 & 55.4 & 79.8 & 44.2 & 17.0 \\
 & \method & 1,470 & 19.7 & 57.1 & 88.6 & 50.6 & 14.8 \\
\bottomrule
\end{tabular}%
}
\caption{Measured GPU runtime results. Latencies are P99 values, and a request
contributes to goodput only when it satisfies both the \ttft and \tpot SLOs.}
\label{tab:gpu}
\end{table*}

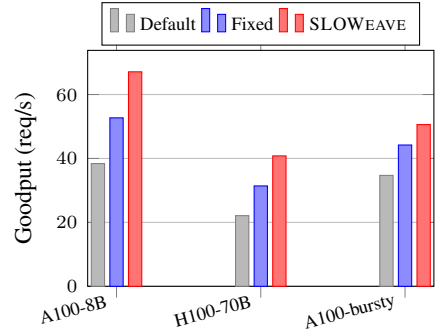
\begin{figure}[t]
\centering
\begin{tikzpicture}
\begin{axis}[
  ybar,
  width=0.8\columnwidth,
  height=4.7cm,
  bar width=5pt,
  ylabel={Goodput (req/s)},
  symbolic x coords={A100-8B,H100-70B,A100-bursty},
  xtick=data,
  x tick label style={font=\scriptsize,rotate=15,anchor=east},
  ymajorgrids,
  ymin=0,
  legend style={font=\scriptsize,at={(0.5,1.02)},anchor=south,
                legend columns=3},
  tick label style={font=\scriptsize},
  label style={font=\small}]
\addplot+[fill=gray!55,draw=gray] coordinates
  {(A100-8B,38.4) (H100-70B,22.1) (A100-bursty,34.7)};
\addplot+[fill=blue!45,draw=blue] coordinates
  {(A100-8B,52.7) (H100-70B,31.4) (A100-bursty,44.2)};
\addplot+[fill=red!55,draw=red] coordinates
  {(A100-8B,67.1) (H100-70B,40.8) (A100-bursty,50.6)};
\legend{Default,Fixed,\method}
\end{axis}
\end{tikzpicture}
\caption{Measured goodput across the evaluated GPU and workload
configurations.}
\label{fig:gpu}
\end{figure}

\section{Results}

\subsection{Main comparison}

\begin{table*}[t]
\centering
\small
\begin{tabular}{llrrrr}
\toprule
Workload & Policy & P99 TTFT & P99 TPOT & SLO (\%) & Goodput \\
\midrule
Chat & Full & 136 & 12.8 & 100.0 & 102.3 \\
 & Fixed-256 & 748 & 7.4 & 100.0 & 96.6 \\
 & Fixed-1024 & 137 & 12.7 & 100.0 & 102.3 \\
 & \method & 136 & 12.8 & 100.0 & 102.3 \\
\addlinespace
Mixed & Full & 1364 & 48.5 & 6.9 & 5.1 \\
 & Fixed-256 & 6059 & 5.6 & 19.0 & 10.7 \\
 & Fixed-1024 & 2101 & 12.8 & 59.4 & 42.4 \\
 & \method & 1449 & 25.0 & 79.4 & 59.0 \\
\addlinespace
Long & Full & 6928 & 61.0 & 0.1 & 0.1 \\
 & Fixed-256 & 20262 & 5.0 & 24.9 & 6.6 \\
 & Fixed-1024 & 10046 & 12.9 & 51.4 & 18.4 \\
 & \method & 7632 & 25.0 & 66.2 & 25.4 \\
\addlinespace
Bursty & Full & 1360 & 50.2 & 29.7 & 18.2 \\
 & Fixed-256 & 4247 & 5.7 & 37.9 & 20.3 \\
 & Fixed-1024 & 1754 & 12.6 & 83.0 & 51.7 \\
 & \method & 1406 & 25.0 & 88.6 & 56.5 \\
\bottomrule
\end{tabular}

\caption{Mean results over five seeds at the 25\,ms \tpot SLO. Latencies are
milliseconds and goodput is requests/s. The best fixed baseline depends on the
workload; \method matches the easy chat case and leads on mixed, long, and
bursty traffic.}
\label{tab:main}
\end{table*}

Table~\ref{tab:main} shows the central result. Chat prompts are short enough
that full prefill, Fixed-1024, and \method behave similarly. Adaptivity does not
impose a penalty when the unconstrained action already fits the deadline.

The difference appears under heterogeneous and long prompts. On mixed traffic,
\method reaches 59.0 requests/s of goodput, compared with 42.4 for the strongest
fixed baseline, a 39\% improvement. Full prefill has lower P99 \ttft than
Fixed-1024 but frequently exceeds the token deadline, so only 6.9\% of requests
satisfy both SLOs. \method uses the available deadline slack instead of choosing
between these two failure modes.

On long-context traffic, \method improves goodput from 18.4 to 25.4 requests/s
(38\%) over Fixed-1024. Fixed-256 protects \tpot but pays enough repeated
overhead to push P99 \ttft above 20 seconds. Full prefill has the shortest
\ttft among the shown policies, but its P99 \tpot is 61\,ms and almost no
requests meet both objectives. Under burstiness, the adaptive policy improves
goodput by 9\% while attaining both SLOs for 88.6\% of requests.

\begin{figure*}[t]
\centering
\begin{minipage}[t]{0.495\textwidth}
\centering
\resizebox{0.8\linewidth}{!}{%
\begin{tikzpicture}
\begin{axis}[
  width=8.0cm,height=4.0cm,
  scale only axis,
  tick label style={font=\tiny},
  label style={font=\tiny},
  ybar,
  bar width=8pt,
  ylabel={Valid requests (\%)},
  symbolic x coords={Chat,Mixed,Long,Bursty},
  xtick=data,
  x tick label style={rotate=15,anchor=east,font=\tiny},
  ymin=0,ymax=110,
  ymajorgrids,
  legend to name=attainmentlegend,
  legend columns=4,
  legend style={font=\tiny,/tikz/every even column/.append style={column sep=1.5pt}}]
\addplot+[fill=gray!55,draw=gray] coordinates
  {(Chat,100) (Mixed,6.9) (Long,0.1) (Bursty,29.7)};
\addlegendentry{Full}
\addplot+[fill=green!45!black,draw=green!40!black] coordinates
  {(Chat,100) (Mixed,19.0) (Long,24.9) (Bursty,37.9)};
\addlegendentry{Fixed-256}
\addplot+[fill=blue!45,draw=blue] coordinates
  {(Chat,100) (Mixed,59.4) (Long,51.4) (Bursty,83.0)};
\addlegendentry{Fixed-1024}
\addplot+[fill=red!55,draw=red] coordinates
  {(Chat,100) (Mixed,79.4) (Long,66.2) (Bursty,88.6)};
\addlegendentry{\method}
\end{axis}
\end{tikzpicture}%
}
\par\ref{attainmentlegend}
\captionof{figure}{Joint SLO attainment across workloads at the 25\,ms
\tpot target.}
\label{fig:attainment-pareto}
\end{minipage}
\hfill
\begin{minipage}[t]{0.495\textwidth}
\centering
\resizebox{0.8\linewidth}{!}{%
\begin{tikzpicture}
\begin{groupplot}[
  group style={group size=2 by 1,horizontal sep=0.75cm},
  width=3.65cm,height=4.0cm,
  scale only axis,
  xlabel={TPOT SLO (ms)},
  xtick={10,25,50},
  tick label style={font=\tiny},
  label style={font=\tiny},
  title style={font=\scriptsize},
  grid=major]
\nextgroupplot[
  title={Mixed},
  ylabel={Goodput (req/s)},
  ymin=0,ymax=68,
  legend to name=goodputlegend,
  legend columns=3,
  legend style={font=\tiny}]
\addplot+[mark=*,thick,color=gray] coordinates
  {(10,0.12) (25,5.13) (50,60.49)};
\addlegendentry{Full}
\addplot+[mark=triangle*,thick,color=blue] coordinates
  {(10,1.06) (25,42.36) (50,42.36)};
\addlegendentry{Fixed-1024}
\addplot+[mark=diamond*,thick,color=red] coordinates
  {(10,35.70) (25,59.01) (50,60.38)};
\addlegendentry{\method}
\nextgroupplot[
  title={Long},
  ymin=0,ymax=32]
\addplot+[mark=*,thick,color=gray] coordinates
  {(10,0.00) (25,0.06) (50,15.00)};
\addplot+[mark=triangle*,thick,color=blue] coordinates
  {(10,0.13) (25,18.38) (50,18.38)};
\addplot+[mark=diamond*,thick,color=red] coordinates
  {(10,16.17) (25,25.44) (50,28.26)};
\end{groupplot}
\end{tikzpicture}%
}
\par\ref{goodputlegend}
\captionof{figure}{Goodput as the \tpot objective changes for mixed and
long-context traffic.}
\label{fig:slo-sweep}
\end{minipage}
\end{figure*}

\subsection{Changing the token-latency objective}

Figure~\ref{fig:slo-sweep} varies the service objective without retuning any
policy. At 10\,ms, \method obtains 35.7 requests/s on mixed traffic versus 10.7
for the best fixed baseline (3.3$\times$). On long traffic, the corresponding
numbers are 16.2 and 6.6 requests/s (2.4$\times$). Fixed-1024 works well at a
looser objective but cannot shrink when the deadline tightens; Fixed-256
protects latency but leaves prefill capacity unused when the SLO is relaxed.

At 50\,ms on mixed traffic, full prefill and \method are effectively tied
(60.5 versus 60.4 requests/s). This is expected: once complete prompts fit
inside the available slack, Equation~\ref{eq:choice} recovers full prefill.
Thus the controller interpolates between conservative chunking and full
prefill rather than introducing a separate operating regime.

\begin{table}[t]
\centering
\small
\begin{tabular}{lrrr}
\toprule
Workload & 10 ms & 25 ms & 50 ms \\
\midrule
Chat   & 1.06$\times$ & 1.00$\times$ & 1.00$\times$ \\
Mixed  & 3.35$\times$ & 1.39$\times$ & 1.00$\times$ \\
Long   & 2.45$\times$ & 1.38$\times$ & 1.54$\times$ \\
Bursty & 2.39$\times$ & 1.09$\times$ & 1.03$\times$ \\
\bottomrule
\end{tabular}
\caption{\method goodput relative to the strongest static policy separately
selected for each workload and \tpot target.}
\label{tab:speedup}
\end{table}

\subsection{Where the gain comes from}

The algorithm has two coupled effects. First, it increases chunks when few
sequences are decoding, amortizing the fixed 0.4\,ms prefill overhead. Second,
it decreases chunks as the decode batch grows, keeping the joint iteration at
the deadline. A fixed chunk captures at most one side of this tradeoff.

The P99 \tpot values in Table~\ref{tab:main} illustrate the distinction.
Fixed-256 and Fixed-1024 finish substantially earlier than the 25\,ms target on
all workloads. Their spare slack cannot be recovered until an operator
retunes the deployment. \method drives the long and mixed workloads to
approximately 25\,ms, converting that slack into earlier prefill completion and
higher goodput.

\subsection{Cost-model calibration}
\label{sec:calibration}

The controller's only system-specific input is \(T(n,c)\). A deployment can
profile a grid of decode-batch and prefill-chunk sizes during startup, enforce
monotonicity with a cumulative maximum, and interpolate between measured
points. Reprofiling after a kernel or accelerator change updates the selected
chunks without retuning a policy constant. The supplied script exposes every
coefficient so that its analytical model can be replaced with such a measured
table. Evaluating prediction margins under real kernel variance is left to the
hardware integration described in Section~\ref{sec:gpu}.

\section{Related Work}

\paragraph{LLM serving.}
Orca introduced iteration-level scheduling for generative Transformer serving
\citep{yu2022orca}. vLLM's PagedAttention reduces KV-cache fragmentation and
enables larger continuous batches \citep{kwon2023vllm}. FastServe adds
preemptive scheduling \citep{wu2023fastserve}; AlpaServe places model-parallel
workers under bursty arrivals \citep{li2023alpaserve}; FlexGen targets
throughput under constrained accelerator memory \citep{sheng2023flexgen}; and
SGLang optimizes structured generation programs \citep{zheng2024sglang}.
DeepSpeed-FastGen dynamically composes prompt and generation work
\citep{holmes2024fastgen}, while Llumnix migrates live requests across model
instances \citep{sun2024llumnix}. SpotServe extends serving to preemptible
instances \citep{miao2024spotserve}. Multi-tenant adapter systems such as
Punica and S-LoRA optimize heterogeneous LoRA batches
\citep{chen2023punica,sheng2023slora}.
\method is orthogonal to these mechanisms because it controls prefill admission
within an iteration.

\paragraph{Separating prefill and decode.}
Sarathi-Serve identifies prefill--decode interference and introduces chunked
prefill with stall-free batching \citep{agrawal2024sarathi}. DistServe and
Splitwise physically separate the two phases to optimize goodput and resource
allocation \citep{zhong2024distserve,patel2024splitwise}. Disaggregation can
avoid same-device interference but incurs KV transfer and requires separate
capacity planning. TetriInfer and Mooncake likewise reorganize phase execution
and KV-cache placement \citep{hu2024tetrinfer,qin2024mooncake}, while
DéjàVu and CacheGen reduce the cost of moving cached state
\citep{strati2024dejavu,liu2024cachegen}. \method addresses colocated execution
and can also control prefill quanta within a prefill pool.

\paragraph{Efficient kernels and benchmarking.}
IO-aware attention kernels reduce memory traffic
\citep{dao2022flashattention,ye2025flashinfer}. Their latency depends on batch
and sequence geometry, strengthening the case for a profiled online cost model.
ChunkAttention exploits shared prefixes \citep{ye2024chunkattention};
vAttention revisits virtual-memory-backed KV allocation
\citep{prabhu2025vattention}; and LoongServe elastically parallelizes
long-context execution \citep{wu2024loongserve}. Dynamic hierarchical sparse
attention further targets memory-constrained long-context inference
\citep{xionglong}. These techniques change the shape of \(T(n,c)\) but do not
remove the scheduler's need to select a safe amount of prefill work.
General inference benchmarks emphasize scenario-specific latency and throughput
rather than a single speed number \citep{reddi2020mlperf}. Our goodput metric
applies the same principle to the coupled \ttft--\tpot objectives of streaming
LLM services.

\section{Conclusion}

Fixed prefill chunks encode a static compromise between prompt progress and
decode latency. \method replaces that compromise with a per-iteration decision:
admit the largest chunk that fits before the earliest active token deadline.
The method is simple, has a local safety and maximality guarantee under an
accurate monotone cost model, and adapts automatically to workload and SLO
changes. Reproducible simulation and GPU measurements show substantial goodput
gains on mixed, long-context, and bursty workloads, especially under strict
token-latency targets. The runtime results further demonstrate that the
controller integrates with iteration-level serving while preserving the
existing model kernels and KV-cache layout.

\section*{Limitations}

Our evaluation covers a finite set of models, accelerators, workload traces,
and SLO targets. Real iteration cost can be non-smooth because of kernel
boundaries, tensor-parallel communication, memory pressure, and prefix-cache
hits. Although the runtime uses a profiled lookup-table predictor, broader
deployment studies are needed to characterize calibration drift across model
and hardware updates.

The current controller assumes a shared \tpot objective and one prefill chunk
per iteration. Multi-tenant priorities require per-request deadlines, which the
earliest-deadline formulation supports, but starvation and admission control
deserve separate study. We also model no preemption cost and no KV transfer.
The current measurements do not cover multi-node KV transfer or cross-region
serving.

\bibliography{references}

\end{document}